\documentclass[runningheads]{llncs}
\usepackage{graphicx}
\graphicspath{{figs/}}
\usepackage{hyperref}

\usepackage{bm}
\usepackage{amsfonts}
\usepackage{amsmath}
\usepackage{lipsum}
\usepackage{algorithm}
\usepackage{algorithmic}

\usepackage{color}
\usepackage{threeparttable}
\usepackage[misc]{ifsym}

\begin{document}
\title{Learning and Exploiting Interclass Visual Correlations for Medical Image Classification}
%
\titlerunning{Learning Interclass Visual Correlations for Medical Image Classification}
%
%
\author{Dong Wei\textsuperscript{(\Letter)} \and
Shilei Cao \and
Kai Ma \and
Yefeng Zheng}
%
%
\authorrunning{D. Wei \textit{et al.}}
%
\institute{Tencent Jarvis Lab, Shenzhen, China\\
\email{\{donwei,eliasslcao,kylekma,yefengzheng\}@tencent.com}}
\maketitle              
\begin{abstract}
Deep neural network-based medical image classifications often use ``hard'' labels for training, where the probability of the correct category is 1 and those of others are 0.
However, these hard targets can drive the networks over-confident about their predictions and prone to overfit the training data, affecting model generalization and adaption.
Studies have shown that label smoothing and softening can improve classification performance.
Nevertheless, existing approaches are either non-data-driven or limited in applicability.
In this paper, we present the Class-Correlation Learning Network (CCL-Net) to learn interclass visual correlations from given training data, and produce soft labels to help with classification tasks.
Instead of letting the network directly learn the desired correlations,
we propose to learn them implicitly via distance metric learning of class-specific embeddings with a lightweight plugin CCL block.
An intuitive loss based on a geometrical explanation of correlation is designed for bolstering learning of the interclass correlations.
We further present end-to-end training of the proposed CCL block as a plugin head together with the classification backbone while generating soft labels on the fly.
Our experimental results on the International Skin Imaging Collaboration 2018 dataset demonstrate effective learning of the interclass correlations from training data, as well as consistent improvements in performance upon several widely used modern network structures with the CCL block.

\end{abstract}
\section{Introduction}
Computer-aided diagnosis (CAD) has important applications in medical image analysis, such as disease diagnosis and grading~\cite{doi2005CAD,doi2007CAD}.
Benefiting from the progress of deep learning techniques, automated CAD methods have advanced remarkably in recent years, and are now dominated by learning-based classification methods using deep neural networks~\cite{ker2017survey,litjens2017survey,shen2017survey}.
Notably, these methods mostly use ``hard'' labels as their targets for learning, where the probability for the correct category is 1 and those for others are 0.
However, these hard targets may adversely affect model generalization and adaption as the networks become over-confident about their predictions when trained to produce extreme values of 0 or 1, and prone to overfit the training data \cite{szegedy2016rethinking}.

Studies showed that label smoothing regularization (LSR) can improve classification performance to a limited extent~\cite{muller2019whenDoesLSHelp,szegedy2016rethinking}, although the smooth labels obtained by uniformly/proportionally distributing probabilities do not represent genuine interclass relations in most circumstances.
Gao \emph{et al.}~\cite{gao2017DLDL} proposed to convert the image label to a discrete label distribution
and use deep convolutional networks to learn from ground truth label distributions.
However, the ground truth distributions were constructed based on empirically defined rules, instead of data-driven.
Chen \emph{et al.}~\cite{chen2019GCN} modeled the interlabel dependencies as a correlation matrix for graph convolutional network based multilabel classification, by mining pairwise cooccurrence patterns of labels.
Although data-driven, this approach is not applicable to single-label scenarios in which labels do not occur together, restricting its application.
Arguably, softened labels reflecting real data characteristics improve upon the hard labels by more effectively utilizing available training data, as the samples from one specific category can help with training of similar categories~\cite{chen2019GCN,gao2017DLDL}.
For medical images, there exist rich interclass visual correlations, which have a great potential yet largely remain unexploited for this purpose.

In this paper, we present the Class-Correlation Learning Network (CCL-Net) to learn interclass visual correlations from given training data, and utilize the learned correlations to produce soft labels to help with
the classification tasks (Fig.~\ref{fig:network}).
Instead of directly letting the network learn the desired correlations,
we propose implicit learning of the correlations via distance metric learning~\cite{sohn2016DML,zhe2019DML} of class-specific embeddings~\cite{mikolov2013word2vec} with a lightweight plugin CCL block.
A new loss is designed for bolstering learning of the interclass correlations.
We further present end-to-end training of the proposed CCL block together with the classification backbone, while enhancing classification ability of the latter with soft labels generated on the fly.
In summary, our contributions are three folds:
\begin{itemize}
  \item We propose a CCL block for data-driven interclass correlation learning and label softening, based on distance metric learning of class embeddings.
      This block is conceptually simple, lightweight, and can be readily plugged as a CCL head into any mainstream backbone network for classification.
  \item We design an intuitive new loss based on a geometrical explanation of correlation to help with the CCL, and present integrated end-to-end training of the plugged CCL head together with the backbone network.
  \item We conduct thorough experiments on the International Skin Imaging Collaboration (ISIC) 2018 dataset.
      Results demonstrate effective data-driven CCL, and consistent performance improvements upon widely used modern network structures utilizing the learned soft label distributions.
\end{itemize}

\section{Method}\label{sec:method}
\subsubsection{Preliminaries}
Before presenting our proposed CCL block, let us first review the generic deep learning pipeline for single-label multiclass classification problems as preliminaries.
As shown in Fig.~\ref{fig:network}(b), an input $x$ first goes through a feature extracting function $f_1$ parameterized as a network with parameters $\theta_1$, producing a feature vector $\bm{f}$ for classification: $\bm{f}=f_1(x|\theta_1)$, where $\bm{f}\in \mathbb{R}^{n_1\times1}$.
Next, $\bm{f}$ is fed to a generalized fully-connected (fc) layer  $f_c$ (parameterized with $\theta_\mathrm{fc}$)
followed by the softmax function, obtaining the predicted classification probability distribution $\bm{q}^\mathrm{cls}$ for $x$: $\bm{q}^\mathrm{cls}=\mathrm{softmax}\left(f_c(\bm{f}|\theta_\mathrm{fc})\right)$, where $\bm{q}^\mathrm{cls}=[q^\mathrm{cls}_1, q^\mathrm{cls}_2, \ldots, q^\mathrm{cls}_K]^T$, $K$ is the number of classes, and $\sum_kq^\mathrm{cls}_k=1$.
To supervise the training of $f_1$ and $f_c$ (and the learning of $\theta_1$ and $\theta_\mathrm{fc}$ correspondingly), for each training sample $x$, a label $y\in\{1, 2, \ldots, K\}$ is given.
This label is then converted to the one-hot distribution~\cite{szegedy2016rethinking}: $\bm{y}=[\delta_{1,y}, \delta_{2,y}, \ldots, \delta_{K,y}]^T$, where $\delta_{k,y}$ is the Dirac delta function, which equals 1 if $k=y$ and 0 otherwise.
After that, the cross entropy loss $l_\mathrm{CE}$
can be used to compute the loss between $\bm{q}^\mathrm{cls}$ and $\bm{y}$:  $l_\mathrm{CE}(\bm{q}^\mathrm{cls},\bm{y})=-\sum_k\delta_{k,y}\log q_k^\mathrm{cls}$.
Lastly, the loss is optimized by an optimizer (e.g., Adam~\cite{kingma2014adam}), and $\theta_1$ and $\theta_\mathrm{fc}$ are updated by gradient descent algorithms.

\begin{figure}[t]
  \centering
  \includegraphics[height=.354\textwidth, trim={41, 119, 20, 54}, clip]{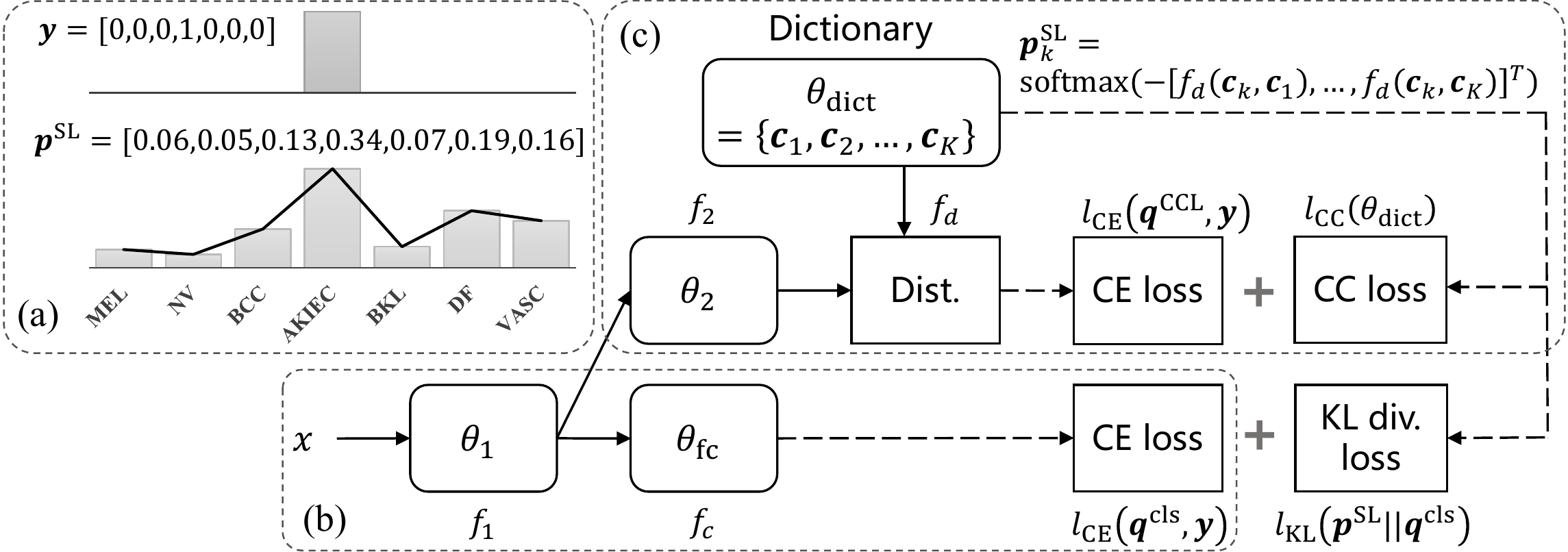}
  \caption{Diagram of the proposed CCL-Net.
  (a) Soft label distributions ($\bm{p}^\mathrm{SL}$) are learned from given training data and hard labels $\bm{y}$.
  (b) Generic deep learning pipeline for classification.
  (c) Structure of the proposed CCL block.
  This lightweight block can be plugged into any classification backbone network as a head and trained end-to-end together, and boost performance by providing additional supervision signals with $\bm{p}^\mathrm{SL}$.}\label{fig:network}
\end{figure}

As mentioned earlier, the ``hard'' labels may adversely affect model generalization as the networks become over-confident about their predictions.
In this sense, LSR~\cite{szegedy2016rethinking} is a popular technique to make the network less confident by smoothing the one-hot label distribution $\bm{y}$ to become $\bm{p}^\mathrm{LSR}=[p^\mathrm{LSR}_{1,y}, \ldots, p^\mathrm{LSR}_{K,y}]^T$, where
$p^\mathrm{LSR}_{k,y}=(1-\epsilon)\delta_{k,y} + \epsilon u(k)$,
$\epsilon$ is a weight, and $u(k)$ is a distribution over class labels
for which the uniform~\cite{muller2019whenDoesLSHelp} or \emph{a priori}~\cite{muller2019whenDoesLSHelp,szegedy2016rethinking} distributions are proposed.
Then, the cross entropy is computed with $\bm{p}^\mathrm{LSR}$ instead of $\bm{y}$.

\subsubsection{Learning Interclass Visual Correlations for Label Softening}
In most circumstances, LSR cannot reflect the genuine interclass relations underlying the given training data.
Intuitively, the probability redistribution should be biased towards visually similar classes, so that samples from these classes can boost training of each other.
For this purpose, we propose to learn the underlying visual correlations among classes from the training data and produce soft label distributions that more authentically reflect intrinsic data properties.
Other than learning the desired correlations directly, we learn them implicitly by learning interrelated yet discriminative class embeddings via distance metric learning.
Both the concepts of feature embeddings and deep metric learning have proven useful in the literature (e.g., \cite{mikolov2013word2vec,sohn2016DML,zhe2019DML}).
To the best of our knowledge, however, combining them for data-driven learning of interclass visual correlations and label softening has not been done before.

The structure of the CCL block is shown in Fig.~\ref{fig:network}(c), which consists of a lightweight embedding function $f_2$ (parameterized with $\theta_2$), a dictionary $\theta_\mathrm{dict}$, a distance metric function $f_d$, and two loss functions.
Given the feature vector $\bm{f}$ extracted by $f_1$, the embedding function $f_2$ projects $\bm{f}$ into the embedding space: $\bm{e}=f_2(\bm{f}|\theta_2)$, where $\bm{e}\in \mathbb{R}^{n_2\times1}$.
The dictionary maintains all the class-specific embeddings: $\theta_\mathrm{dict}=\{\bm{c}_k\}$.
Using $f_d$, the distance between the input embedding and every class embedding can be calculated by $d_k=f_d(\bm{e}, \bm{c}_k)$.
In this work, we use $f_d(\bm{e}_1,\bm{e}_2)=\left\| \frac{\bm{e}_1}{\|\bm{e}_1\|} - \frac{\bm{e}_2}{\|\bm{e}_2\|} \right\|^2$, where $\|\cdot\|$ is the L2 norm.
Let $\bm{d}=[d_1, d_2, \ldots, d_K]^T$, we can predict another classification probability distribution $\bm{q}^\mathrm{CCL}$ based on the distance metric: $\bm{q}^\mathrm{CCL}=\mathrm{softmax}(-\bm{d})$, and a cross entropy loss $l_\mathrm{CE}(\bm{q}^\mathrm{CCL}, \bm{y})$ can be computed accordingly.
To enforce interrelations among the class embeddings, we innovatively propose the class correlation loss $l_\mathrm{CC}$:
\begin{equation}\label{eq:l_cc}
  l_\mathrm{CC}(\theta_\mathrm{dict}) = 1/K^2{\sum}^K_{k_1=1}{\sum}^K_{k_2=1}\Big|f_d(\bm{c}_{k_1}, \bm{c}_{k_2}) - b\Big|_+
\end{equation}
where $|\cdot|_+$ is the Rectified Linear Unit (ReLU), and $b$ is a margin parameter.
Intuitively, $l_\mathrm{CC}$ enforces the class embeddings to be no further than a distance $b$ from each other in the embedding space, encouraging correlations among them.
Other than attempting to tune the exact value of $b$, we resort to the geometrical meaning of correlation between two vectors:
if the angle between two vectors is smaller than 90$^{\circ}$, they are considered correlated.
Or equivalently for L2-normed vectors, if the Euclidean distance between them is smaller than $\sqrt{2}$, they are considered correlated.
Hence, we set $b=(\sqrt{2})^2$.
Then, the total loss function for the CCL block is defined as
$\mathcal{L}_\mathrm{CCL}=l_\mathrm{CE}(\bm{q}^\mathrm{CCL}, \bm{y}) + \alpha_\mathrm{CC} l_\mathrm{CC}(\theta_\mathrm{dict})$,
where $\alpha_\mathrm{CC}$ is a weight.
Thus, $\theta_2$ and $\theta_\mathrm{dict}$ are updated by optimizing $\mathcal{L}_\mathrm{CCL}$.

After training, we define the soft label distribution $\bm{p}^\mathrm{SL}_k$ for class $k$ as:
\begin{equation}\label{eq:soft_label}
  \bm{p}^\mathrm{SL}_k = \mathrm{softmax}(-[f_d(\bm{c}_1, \bm{c}_k), \ldots, f_d(\bm{c}_K, \bm{c}_k)]^T) = [p^\mathrm{SL}_{1,k}, \ldots, p^\mathrm{SL}_{K,k}]^T.
\end{equation}
It is worth noting that by this definition of soft label distributions, the correct class always has the largest probability, which is a desired property especially at the start of training.
Next, we describe our end-to-end training scheme of the CCL block together with a backbone classification network.

\subsubsection{Integrated End-to-End Training with Classification Backbone}
\label{sec:e2e_train}
As a lightweight module, the proposed CCL block can be plugged into any mainstream classification backbone network---as long as a feature vector can be pooled for an input image---and trained together in an end-to-end manner.
To utilize the learned soft label distributions, we introduce a Kullback-Leibler divergence (KL div.) loss
$l_\mathrm{KL}(\bm{p}^\mathrm{SL}||\bm{q}^\mathrm{cls}) = \sum_{k}p^\mathrm{SL}_k\log (p^\mathrm{SL}_k/q^\mathrm{cls}_k)$
in the backbone network (Fig~\ref{fig:network}), and the total loss function for classification becomes
\begin{equation}\label{eq:l_cls}
  \mathcal{L}_\mathrm{cls} = l_\mathrm{CE}(\bm{q}^\mathrm{cls}, \bm{y}) + l_\mathrm{KL}(\bm{p}^\mathrm{SL}||\bm{q}^\mathrm{cls}).
\end{equation}
Consider that $l_\mathrm{CC}$ tries to keep the class embeddings within a certain distance of each other, it is somehow adversarial to
the goal of
the backbone network which tries to push them away from each other as much as possible.
In such cases, alternative training schemes are usually employed~\cite{goodfellow2014generative} and we follow this way.
Briefly, in each training iteration, the backbone network is firstly updated with the CCL head frozen, and then it is frozen to update the CCL head;
more details about the training scheme are provided in Algorithm~\ref{alg:training}.
After training, the prediction is made according to $\bm{q}^\mathrm{cls}$: $\hat{y}=\mathrm{argmax}_k(q^\mathrm{cls}_k)$.

\begin{algorithm}[!t]
  \caption{End-to-end training of the proposed CCL-Net.}\label{alg:training}
  \begin{algorithmic}[1]
  \REQUIRE Training images $\{x\}$ and labels $\{y\}$
  \ENSURE Learned network parameters $\{\theta_1, \theta_\mathrm{fc}, \theta_2, \theta_\mathrm{dict}\}$
  \STATE {Initialize $\theta_1, \theta_\mathrm{fc}, \theta_2, \theta_\mathrm{dict}$}
  \FOR {number of training epochs}
    \FOR {number of minibatches}
    \STATE Compute soft label distributions $\{\bm{p}^\mathrm{SL}\}$ from $\theta_\mathrm{dict}$
    \STATE Sample minibatch of $m$ images $\{x^{(i)}|i\in\{1,\ldots,m\}\}$, compute $\{\bm{f}^{(i)}\}$
    \STATE Update $\theta_1$ and $\theta_\mathrm{fc}$ by stochastic gradient descending: $\nabla_{\{\theta_1, \theta_\mathrm{fc}\}} \frac{1}{m} \sum_{i=1}^{m} \mathcal{L}_\mathrm{cls}$
    \STATE Update $\theta_2$ and $\theta_\mathrm{dict}$ by stochastic gradient descending:$\nabla_{\{\theta_2, \theta_\mathrm{dict}\}} \frac{1}{m} \sum_{i=1}^{m} \mathcal{L}_\mathrm{CCL}$
    \ENDFOR
  \ENDFOR
  \end{algorithmic}
\end{algorithm}

During training, we notice that the learned soft label distributions sometimes start to linger around a fixed value for the correct classes and distribute about evenly across other classes after certain epochs;
or in other words, approximately collapse into LSR with $u(k)=$ uniform distribution (the proof is provided in the supplementary material).
This is because of the strong capability of deep neural networks in fitting any data~\cite{zhang2016understanding}.
As $l_\mathrm{CC}$ forces the class embeddings to be no further than $b$ from each other in the embedding space, a network of sufficient capacity has the potential to make them exactly the distance $b$ away from each other (which is overfitting) when well trained.
To prevent this from happening, we use the average correct-class probability $\bar{p}^\mathrm{SL}_{k,k}=1/K\sum_{k}p^\mathrm{SL}_{k,k}$ as a measure of the total softness of the label set (the lower the softer, as the correct classes distribute more probabilities to other classes),
and consider that the CCL head has converged if $\bar{p}^\mathrm{SL}_{k,k}$ does not drop for 10 consecutive epochs.
In such case, $\theta_\mathrm{dict}$ is frozen, whereas the rest of the CCL-Net keeps updating.

\section{Experiments}
\subsubsection{Dataset and Evaluation Metrics}
The ISIC 2018 dataset is provided by the Skin Lesion Analysis Toward Melanoma Detection 2018 challenge~\cite{codella2019ISIC},
for prediction of seven
disease categories with dermoscopic images, including: melanoma (MEL), melanocytic nevus (NV), basal cell carcinoma (BCC), actinic keratosis/Bowen’s disease (AKIEC), benign keratosis (BKL), dermatofibroma (DF), and vascular lesion (VASC)
(example images are provided in Fig.~\ref{fig:data}).
It comprises 10,015 dermoscopic images,
including 1,113 MEL, 6,705 NV, 514 BCC, 327 AKIEC, 1,099 BKL, 115 DF, and 142 VASC images.
We randomly split the data into a training and a validation set (80:20 split) while keeping the original interclass ratios, and report evaluation results on the validation set.
The employed evaluation metrics include accuracy, Cohen's kappa coefficient~\cite{cohen1960coefficient}, and unweighed means of F1 score and Jaccard similarity coefficient.

\begin{figure}[t]
  \centering
  \begin{minipage}[c]{.135\textwidth}
    \centering
    \includegraphics[width=\linewidth]{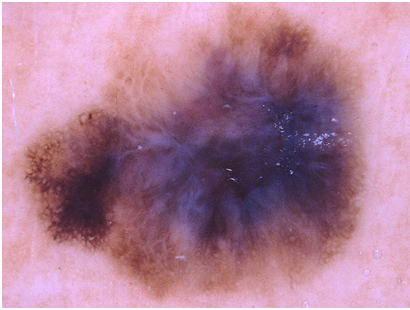}\\
    \scriptsize MEL
  \end{minipage}
  \begin{minipage}[c]{.135\textwidth}
    \centering
    \includegraphics[width=\linewidth]{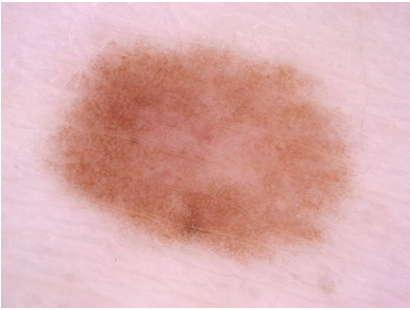}\\
    \scriptsize NV
  \end{minipage}
  \begin{minipage}[c]{.135\textwidth}
    \centering
    \includegraphics[width=\linewidth]{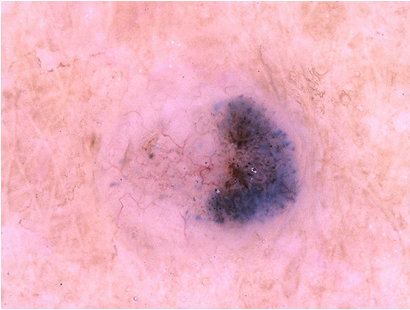}\\
    \scriptsize BCC
  \end{minipage}
  \begin{minipage}[c]{.135\textwidth}
    \centering
    \includegraphics[width=\linewidth]{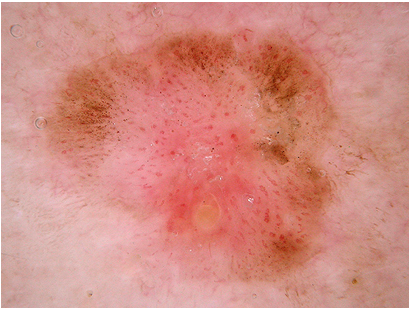}\\
    \scriptsize AKIEC
  \end{minipage}
  \begin{minipage}[c]{.135\textwidth}
    \centering
    \includegraphics[width=\linewidth]{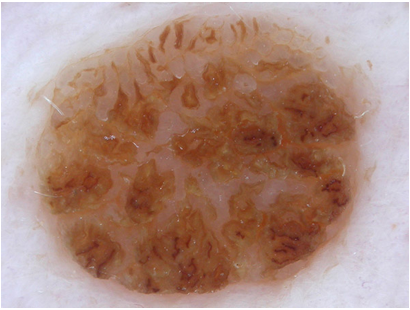}\\
    \scriptsize BKL
  \end{minipage}
  \begin{minipage}[c]{.135\textwidth}
    \centering
    \includegraphics[width=\linewidth]{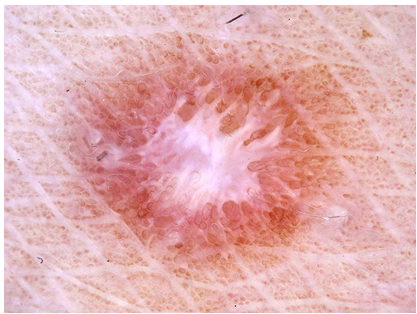}\\
    \scriptsize DF
  \end{minipage}
  \begin{minipage}[c]{.135\textwidth}
    \centering
    \includegraphics[width=\linewidth]{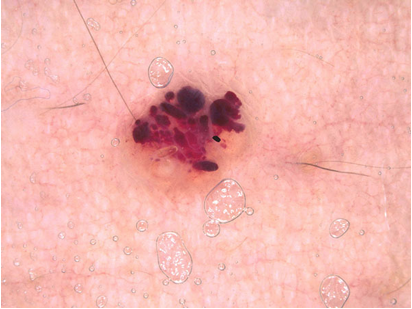}\\
    \scriptsize VASC
  \end{minipage}
  \caption{Official example images of the seven diseases in the ISIC 2018 disease classification dataset~\cite{codella2019ISIC}.}\label{fig:data}
\end{figure}

\subsubsection{Implementation}
In this work, we use commonly adopted, straightforward training schemes to demonstrate effectiveness of the proposed CCL-Net in data-driven learning of interclass visual correlations and improving classification performance, rather than sophisticated training strategies or heavy ensembles.
Specifically, we adopt the stochastic gradient descent optimizer with a momentum of 0.9, weight decay of $10^{-4}$, and the backbone learning rate initialized to 0.1 for all experiments.
The learning rate for the CCL head (denoted by $lr_\mathrm{CCL}$) is initialized to 0.0005 for all experiments,
except for when we study the impact of varying $lr_\mathrm{CCL}$.
Following~\cite{he2016resnet}, we multiply the initial learning rates by 0.1 twice during training such that the learning process can saturate at a higher limit.\footnote{The exact learning-rate-changing epochs as well as total number of training epochs vary for different backbones due to different network capacities.}
A minibatch of 128 images is used.
The input images are resized to have a short side of 256 pixels while maintaining original aspect ratios.
Online data augmentations including random cropping, horizontal and vertical flipping, and color jittering
are employed for training.
For testing, a single central crop of size 224$\times$224 pixels is used as input.
Gradient clipping is employed for stable training.
$\alpha_\mathrm{CC}$ is set to 10 empirically.
The experiments are implemented using the PyTorch package.
A singe Tesla P40 GPU is used for model training and testing.
For $f_2$, we use three fc layers of width 1024, 1024, and 512 (i.e., $n_2=512$), with batch normalization and ReLU in between.

\begin{table}[t]
\caption{Experimental results on the ISIC18 dataset, including comparisons with the baseline networks and LSR~\cite{szegedy2016rethinking}.
Higher is better for all evaluation metrics.}\label{tab:ISIC18}
\scalebox{.965}{
\centering
\begin{threeparttable}
\setlength{\tabcolsep}{2.25mm}{
\begin{tabular}{llccccccc}
\hline\hline
Backbone & Method & Accuracy & F1 score & Kappa & Jaccard \\
\hline
ResNet-18~\cite{he2016resnet}
          & Baseline & 0.8347 & 0.7073 & 0.6775 & 0.5655 \\
          & LSR-u1 & 0.8422 & 0.7220 & 0.6886 & 0.5839 \\
          & LSR-u5 & 0.8437 & 0.7211 & 0.6908 & 0.5837 \\
          & LSR-a1 & 0.7883 & 0.6046 & 0.5016 & 0.4547 \\
          & LSR-a5 & 0.7029 & 0.2020 & 0.1530 & 0.1470 \\
          & CCL-Net (ours) & \textbf{0.8502} & \textbf{0.7227} & \textbf{0.6986} & \textbf{0.5842} \\
\hline
EfficientNet-B0~\cite{tan2019efficientnet}
          & Baseline & 0.8333 & 0.7190 & 0.6696 & 0.5728 \\
          & LSR-u1 & 0.8382 & 0.7014 & 0.6736 & 0.5573 \\
          & LSR-u5 & 0.8432 & 0.7262 & 0.6884 & 0.5852 \\
          & LSR-a1 & 0.8038 & 0.6542 & 0.5526 & 0.5058 \\
          & LSR-a5 & 0.7189 & 0.2968 & 0.2295 & 0.2031 \\
          & CCL-Net (ours) & \textbf{0.8482} & \textbf{0.7390} & \textbf{0.6969} & \textbf{0.6006} \\
\hline
MobileNetV2~\cite{sandler2018mobilenetv2}
          & Baseline & 0.8308 & 0.6922 & 0.6637 & 0.5524 \\
          & LSR-u1 & 0.8248 & 0.6791 & 0.6547 & 0.5400 \\
          & LSR-u5 & 0.8253 & 0.6604 & 0.6539 & 0.5281 \\
          & LSR-a1 & 0.8068 & 0.6432 & 0.5631 & 0.4922 \\
          & LSR-a5 & 0.7114 & 0.2306 & 0.2037 & 0.1655 \\
          & CCL-Net (ours) & \textbf{0.8342} & \textbf{0.7050} & \textbf{0.6718} & \textbf{0.5648} \\
\hline\hline
\end{tabular}}
\begin{tablenotes}
\item[*] LSR settings: -u1: $u(k)=$ uniform, $\epsilon=0.1$;
         -u5: $u(k)=$ uniform, $\epsilon=0.5228$;
         -a1: $u(k)=$ \emph{a priori}, $\epsilon=0.1$;
         -a5: $u(k)=$ \emph{a priori}, $\epsilon=0.5228$.
\end{tablenotes}
\end{threeparttable}}
\end{table}

\subsubsection{Comparison with Baselines and LSR}
We quantitatively compare our proposed CCL-Net with various baseline networks using the same backbones.
Specifically, we experiment with three widely used backbone networks: ResNet-18~\cite{he2016resnet}, MobileNetV2~\cite{sandler2018mobilenetv2}, and EfficientNet-B0~\cite{tan2019efficientnet}.
In addition, we compare our CCL-Net with the popular LSR~\cite{szegedy2016rethinking} with different combinations of $\epsilon$ and $u(k)$: $\epsilon\in\{0.1, 0.5228\}$ and $u(k)\in\{\mathrm{uniform}, a\ priori\}$, resulting in a total of four settings
(we compare with $\epsilon=0.5228$ since for the specific problem, the learned soft label distributions would eventually approximate
uniform LSR with this $\epsilon$ value, if the class embeddings $\theta_\mathrm{dict}$ are not frozen after convergence).
The results are charted in Table~\ref{tab:ISIC18}.
As we can see, the proposed CCL-Net achieves the best performances on all evaluation metrics for all backbone networks, including Cohen's kappa~\cite{cohen1960coefficient} which is more appropriate for imbalanced data than accuracy.
These results demonstrate effectiveness of utilizing the learned soft label distributions in improving classification performance of the backbone networks.
We also note that moderate improvements are achieved by the LSR settings with $u(k)=$ uniform on two of the three backbone networks, indicating effects of this simple strategy.
Nonetheless, these improvements are outweighed by those achieved by our CCL-Net.
In addition to the superior performances to LSR, another advantage of the CCL-Net is that it can intuitively reflect intrinsic interclass correlations underlying the given training data, at a minimal extra overhead.
Lastly, it is worth noting that the LSR settings with $u(k)=$ \emph{a priori} decrease all evaluation metrics from the baseline performances, suggesting inappropriateness of using LSR with \emph{a priori} distributions for significantly imbalanced data.

\begin{table}[t]
\caption{Properties of the CCL by varying the learning rate of the CCL head. ResNet-18~\cite{he2016resnet} backbone is used.
}\label{tab:softness}
\setlength{\tabcolsep}{2.9mm}{
\scalebox{0.965}{
\centering
\begin{tabular}{ccccccc}
  \hline
  $lr_\mathrm{CCL}$ & 0.0001 & 0.0005 & 0.001 & 0.005 & 0.01 & 0.05 \\
  \hline
  Epochs of converge & Never & 137 & 76 & 23 & 12 & 5 \\
  $\bar{p}^\mathrm{SL}_{k,k}$ & 0.48 & 0.41 & 0.38 & 0.35 & 0.33 & 0.30 \\
  Accuracy & 0.8422 & \textbf{0.8502} & 0.8452 & 0.8442 & 0.8422 & 0.8417 \\
  Kappa & 0.6863 & \textbf{0.6986} & 0.6972 & 0.6902 & 0.6856 & 0.6884 \\
  \hline
\end{tabular}}}
\end{table}

\begin{figure}[t]
  \centering
  \includegraphics[height=.29\textwidth]{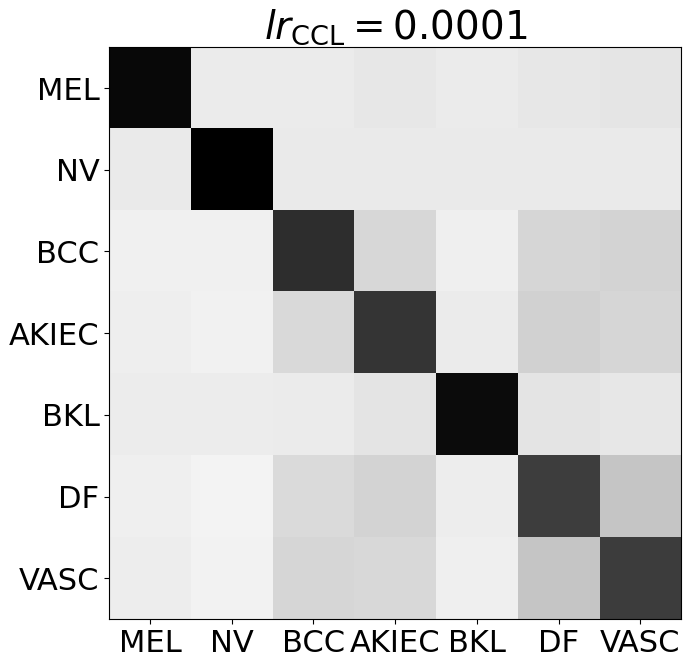}
  \includegraphics[height=.29\textwidth]{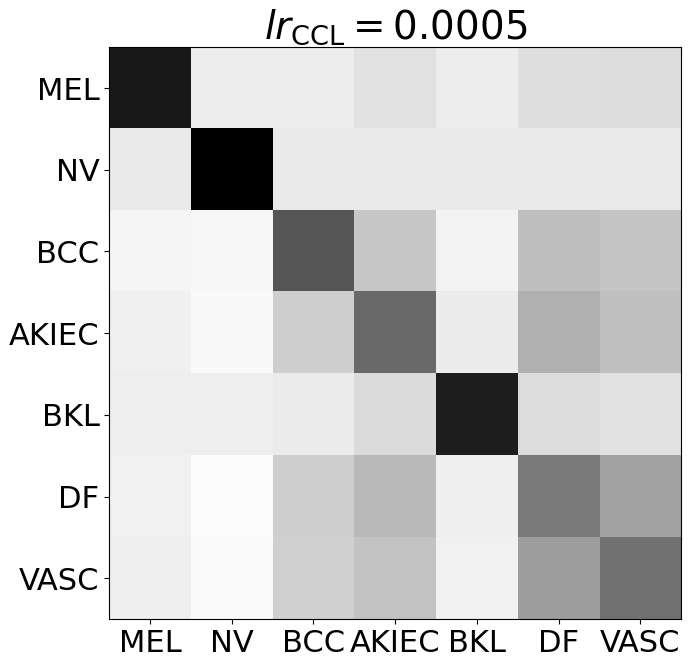}
  \includegraphics[height=.29\textwidth]{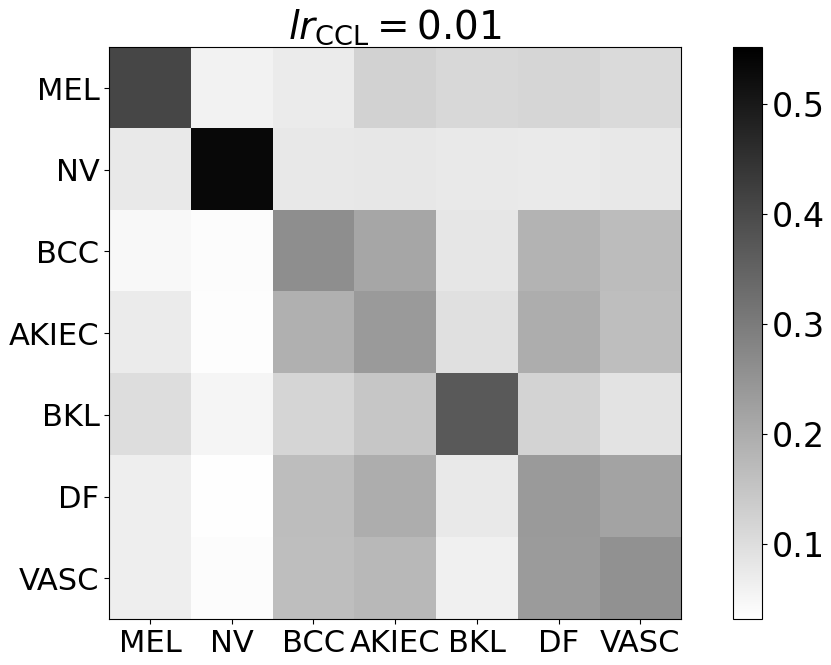}
  \caption{Visualization of the learned soft label distributions using different $lr_\mathrm{CCL}$, where each row of a matrix represents the soft label distribution of a class.}\label{fig:corr_mat}
\end{figure}

\subsubsection{Analysis of Interclass Correlations Learned with CCL-Net}
\label{sec:analysis}
Next, we investigate properties of the learned interclass visual correlations by the proposed CCL-Net, by varying the value of $lr_\mathrm{CCL}$.
Specifically, we examine the epochs and label softness when the CCL head converges, as well as the final accuracies and kappa coefficients.
Besides $\bar{p}_{k,k}^\mathrm{SL}$, the overall softness of the set of soft label distributions can also be intuitively perceived by visualizing all $\{\bm{p}_k^\mathrm{SL}\}$ together as a correlation matrix.
Note that this matrix does not have to be symmetric, since the softmax operation is separately conducted for each class.
Table~\ref{tab:softness} presents the results, and Fig.~\ref{fig:corr_mat} shows three correlation matrices using different $lr_\mathrm{CCL}$.
Interestingly, we can observe that as $lr_\mathrm{CCL}$ increases, the CCL head converges faster with higher softness.
The same trend can be observed in Fig.~\ref{fig:corr_mat}, where the class probabilities become more spread with the increase of $lr_\mathrm{CCL}$.
Notably, when $lr_\mathrm{CCL}=0.0001$, the CCL head does not converge in given epochs and the resulting label distributions are not as soft.
This indicates that when $lr_\mathrm{CCL}$ is too small, the CCL head cannot effectively learn the interclass correlations.
Meanwhile, the best performance is achieved when $lr_\mathrm{CCL}=0.0005$ in terms of both accuracy and kappa,
instead of other higher values.
This may suggest that very quick convergence may also be suboptimal, probably because the prematurely frozen class embeddings are learned from the less representative feature vectors in the early stage of training.
In summary, $lr_\mathrm{CCL}$ is a crucial parameter for the CCL-Net,
though it is not difficult to tune based on our experience.

\section{Conclusion}
In this work, we presented CCL-Net for data-driven interclass visual correlation learning and label softening.
Rather than directly learning the desired correlations, CCL-Net implicitly learns them via distance-based metric learning of class-specific embeddings, and constructs soft label distributions from learned correlations by performing softmax on pairwise distances between class embeddings.
Experimental results showed that the learned soft label distributions not only reflected intrinsic interrelations underlying given training data, but also boosted classification performance upon various baseline networks.
In addition, CCL-Net outperformed the popular LSR technique.
We plan to better utilize the learned soft labels and extend the work for multilabel problems in the future.
%
\subsubsection{Acknowledgments.} This work was funded by the Key Area Research and Development Program of Guangdong Province, China (No. 2018B010111001), National Key Research and Development Project (2018YFC2000702), and Science and Technology Program of Shenzhen, China (No. ZDSYS201802021814180).
%
%
%
\bibliographystyle{splncs04}
\bibliography{CCL-Net}
%
%
%
%
%
%
\newpage
\begin{center}
\textbf{\large Supplementary Material: Learning and Exploiting Interclass Visual Correlations for Medical Image Classification}
\end{center}
\setcounter{equation}{0}
\setcounter{figure}{0}
\setcounter{table}{0}
\setcounter{page}{1}
\makeatletter
\renewcommand{\theequation}{S\arabic{equation}}
\renewcommand{\thefigure}{S\arabic{figure}}
\renewcommand{\theproposition}{S\arabic{proposition}}

\begin{proposition}
When the CCL head perfectly overfits, the learned soft label distributions are equivalent to LSR using a definite $\epsilon$ with $u(k)$ being a uniform distribution.
\end{proposition}
\begin{proof}
The CCL head, if capable, may eventually overfit the data by pushing the class embeddings as far as possible, up to the extent constrained by Equation~(\ref{eq:l_cc}).
This results in class embeddings in $\theta_\mathrm{dict}$ equally distant from each other, and the distance is exactly $b$, i.e., $f_d(\bm{c}_{k'}, \bm{c}_k)=b$ if $k’\neq k$ and 0 otherwise.
Applying the softmax operation, for class $k$ we have
\begin{equation}
  p^\mathrm{SL}_{k', k}
  =\frac{\exp(-f_d(\bm{c}_{k'}, \bm{c}_k))}{\sum_{k''}\exp(-f_d(\bm{c}_{k''}, \bm{c}_k))}
  =\begin{cases}
     \exp(-0)/S=1/S, & \mbox{if } k'=k, \\
     \exp(-b)/S=a/S, & \mbox{otherwise},
   \end{cases}
\end{equation}
where $a=\exp(-b)$ and $S=1+(K-1)a$.
Therefore, in the resulted soft label distribution for class $k$, the probability for the correct class is $1/S$, and those for other classes are all $a/S$, equivalent to LSR of $u(k)=1/K$ and $\epsilon=Ka/S$.\qed
\end{proof}

\begin{figure}
  \centering
  \includegraphics[height=.29\textwidth]{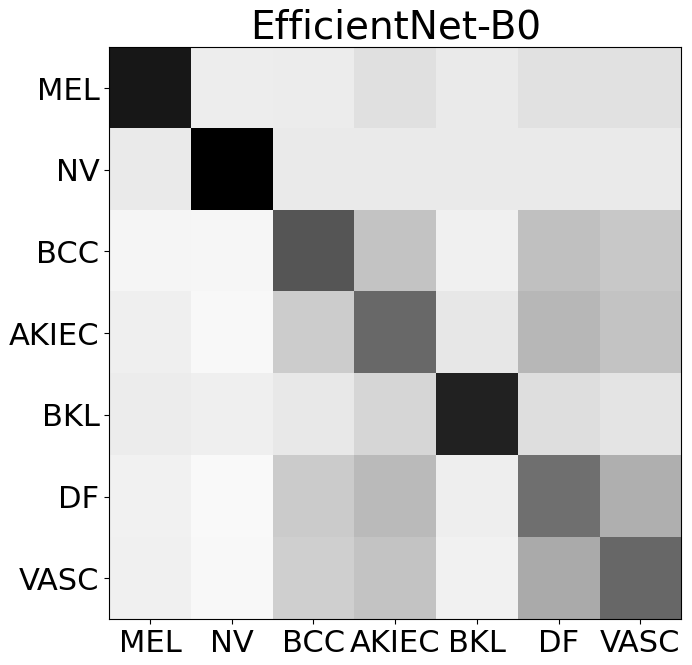}
  \includegraphics[height=.29\textwidth]{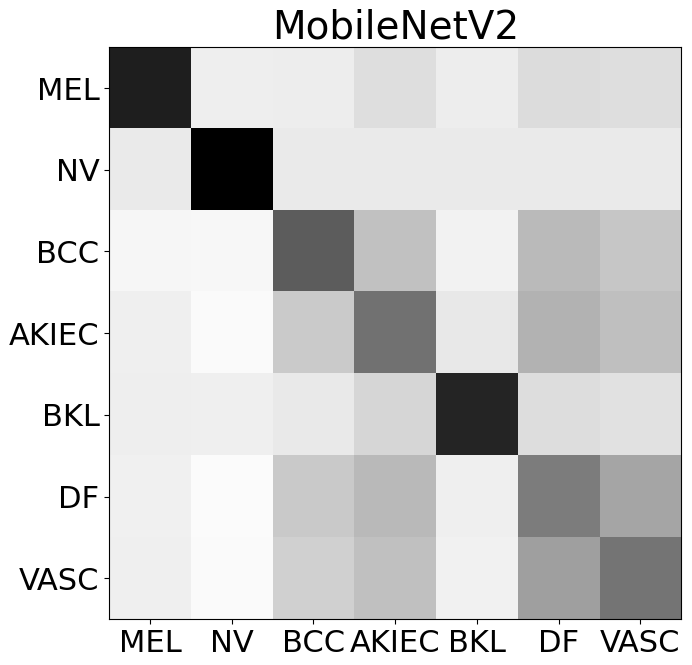}
  \includegraphics[height=.29\textwidth]{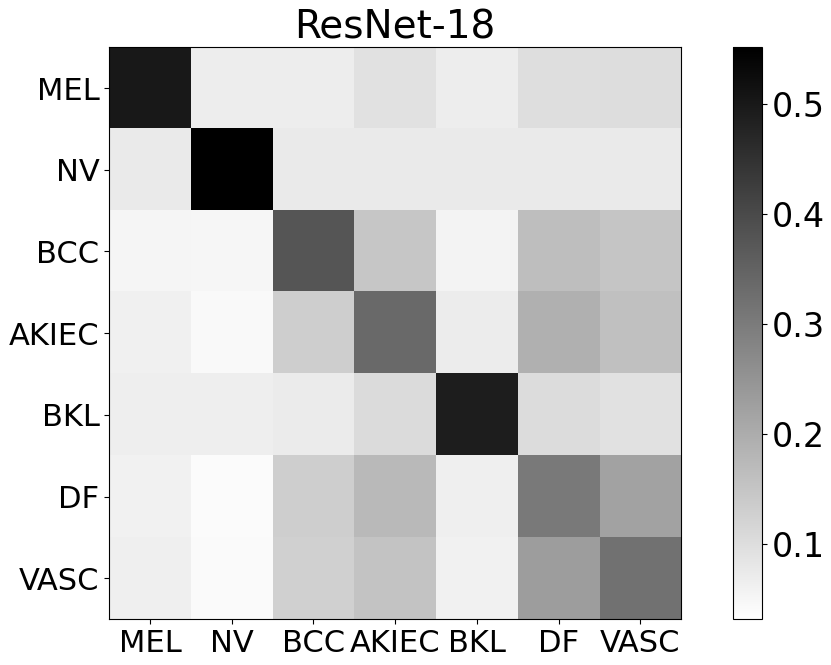}
  \caption{Visualization of the learned soft label distributions using different backbone networks ($lr_\mathrm{CCL}=0.0005$).
  Each row of a matrix represents the soft label distribution of a class.
  From the visualized correlation matrices, we observe that the interclass correlations/soft label distributions learned by different backbone networks are remarkably consistent.}
\end{figure}

\end{document}